\documentclass[letterpaper, 10 pt, conference]{ieeeconf}  
\IEEEoverridecommandlockouts                              

\usepackage{verbatim}
\usepackage{amssymb}
\usepackage{mathtools}
\usepackage{nth}
\usepackage[vlined,ruled]{algorithm2e}
\usepackage{enumerate}

\makeatletter
\let\NAT@parse\undefined
\makeatother

\usepackage[square, sort&compress, numbers]{natbib}

\usepackage{amsmath,graphicx,epsfig,color,amsfonts, subfigure}

\graphicspath{{./fig/}}

\def\expt{\mathbb{E}}
\def\real{\mathbb{R}}
\def\integer{\mathbb{Z}}
\def\natural{\mathbb{N}}
\newcommand{\until}[1]{\{1,\dots, #1\}}

\newcommand{\setdef}[2]{\{#1 \; | \; #2\}}
\newcommand{\seqdef}[2]{\{#1\}_{#2}}

\newcommand{\union}{\operatorname{\cup}}

\DeclareMathOperator*{\argmax}{arg\,max}
\DeclareMathOperator*{\argmin}{arg\,min}

\newcommand\oprocendsymbol{\hbox{$\square$}}
\newcommand\oprocend{\relax\ifmmode\else\unskip\hfill\fi\oprocendsymbol}

\newcommand\bit[1]{\textit{\textbf{#1}}}

\def \bs {\boldsymbol}

\def \etal {\emph{et al.}}

\newtheorem{theorem}{Theorem}

\newtheorem{lemma}[theorem]{Lemma}

\newtheorem{remark}{Remark}

\renewcommand{\theenumi}{(\roman{enumi}}

\title{On Distributed Multi-player Multiarmed Bandit Problems\\ in Abruptly Changing Environment 
\thanks{This work has been supported by NSF Award IIS-1734272.}
}

\author{Lai Wei \hspace{1in} Vaibhav Srivastava
\thanks{L. Wei and V. Srivastava are with the Department of Electrical and Computer Engineering. Michigan State University, East Lansing, MI 48823 USA.
        {\tt\small e-mail: weilai1@msu.edu; e-mail: vaibhav@egr.msu.edu }}%
}

\begin{document}

\maketitle
\thispagestyle{empty}
\pagestyle{empty}

\begin{abstract}
We study the multi-player stochastic multiarmed bandit (MAB) problem in an abruptly changing environment. We consider a collision model in which a player receives reward at an arm if it is the only player to select the arm. We design two novel algorithms, namely, Round-Robin Sliding-Window Upper Confidence Bound\# (RR-SW-UCB\#), and the Sliding-Window Distributed Learning with Prioritization (SW-DLP).   
We  rigorously analyze these algorithms and show that the expected cumulative group regret for these algorithms  is upper bounded by sublinear functions of time, i.e., the time average of the regret asymptotically converges to zero. We complement our analytic results with numerical illustrations. 
\end{abstract}

\section{Introduction}
Achieving coordinated behavior of multiple decision-makers in unknown, uncertain, and non-stationary environments without any explicit communication among them is of immense interest in a variety of applications, including robotic swarming, spectrum access, and Internet of Things. Such decision-making problems often require decision-makers to balance tradeoff between learning the environment and the policies of other decision-makers, and maximizing their own utility.




The distributed  multi-player MAB problem embodies this fundamental decision-making tradeoff. In these problems,  multiple decision-makers sequentially allocate a single resource each by repeatedly choosing one among a set of competing alternative arms (options). The objective of decision-makers is to maximize their total reward, while ensuring that their allocation at each time is not coincident with any other decision-maker. In a non-stationary version of these problems, the utility (reward) of the arms varies with time.  These problems provide a rich modeling framework for a variety of interesting settings, including, robotic foraging and surveillance~\cite{JRK-AK-PT:78,VS-PR-NEL:13, VS-PR-NEL:14}, acoustic relay positioning for underwater communication~\cite{MYC-JL-FSH:13}, and channel allocation in communication networks~\cite{anandkumar2011distributed}.

Most of the algorithmic solutions to the multi-player MAB problem deal with a stationary environment. In~\cite{Multi-playerMABIID,Multi-playerMABMarkovian}, a lower bound on expected cumulative group regret for a centralized policy is derived and algorithms that asymptotically achieve this lower bound are designed.  
Distributed multi-player MAB problem with no communication among players 
has been studied in~\cite{anandkumar2011distributed,gai2014distributed,liu2010distributed,kalathil2014decentralized,NN-DK-RJ:16}. Distributed cooperative multi-player MAB problem in which agents communicate to improve their estimates of mean rewards are studied in~\cite{shahrampour_rakhlin_jadbabaie_2017,arxiv:LandgrenSL15,DBLP:journals/corr/LandgrenSL16}. 





In the context of the single player MAB problem, some classes of non-stationary MAB problem have been studied in the literature. In~\cite{auer2002nonstochastic}, authors study a non-stochastic MAB problem in which the rewards are deterministic and non-stationary. They propose algorithmic solutions to this problem and study a weaker notion of the regret, wherein the policy generated by the algorithm is compared 
against the best policy within the policies that select the same arm at each time. These algorithms  are adapted to handle a class of non-stationary environments and upper bounds on the standard notion of the regret are derived in~\cite{besbes2014optimal}.  In~\cite{AG-EM:08}, authors study a class of non-stationary MAB problems in which the mean rewards at arms may switch abruptly at unknown times to unknown values. They design an upper confidence bound (UCB) based algorithm that relies on estimates of the mean rewards from a recent time-window of observations. In~\cite{liu2017change}, authors study the MAB problem in a piecewise-stationary environment. They use active detection algorithms to determine the change-points and restart the UCB algorithm. Authors in~\cite{LW-VS:17i} develop variants of the algorithm in~\cite{AG-EM:08} for non-stationary MAB problems with abruptly-changing and slowly-varying mean rewards. They also develop deterministic sequencing of exploration and exploitation type of algorithms for these environments.

In this paper, we build upon the literature on distributed multi-player MAB problem with no communication among players and design analogous algorithms for a class of non-stationary environments. In particular, we extend ideas in~\cite{LW-VS:17i} to the multi-player setting and design two algorithms to solve these problems.  
The major contributions of this work are twofold. First, we design and rigorously analyze two algorithms, namely, RR-SW-UCB\# and SW-DLP for the multi-player MAB problem in abruptly changing environments. Second, we augment our analysis with numerical illustrations. 

The paper is organized as follows. In Section~\ref{sec:background}, we discuss preliminaries and introduce the problem. In Section~\ref{sec:RR-SW-UCB} and \ref{sec:SW-DLP}, we design and analyze the RR-SW-UCB\# and SW-DLP algorithms, respectively. We illustrate these algorithms with some numerical examples in Section~\ref{sec:numerical}. Finally, we conclude in Section~\ref{sec:conclusions}.

\section{Background \& Problem Description}\label{sec:background}

In this section, we formulate the muti-player non-stationary MAB problem, and introduce a class of non-stationary environments: the abruptly-changing environment.

\subsection{The non-stationary multi-player stochastic MAB problem}
Consider an $N$-armed bandit problem, i.e., an MAB problem with $N$ arms. Let the total number of players be $M \in \until{N}$. The reward associated with arm $i\in \until{N}$ is modeled as a random variable with bounded support $[0,1]$ and an unknown time-varying mean $\mu_i(t) \in [0,1]$. Each player $j \in \until{M}$ selects a particular arm $s_j(t) \in \until{N}$ at time $t \in \until{T}$ and observes a reward $r_{s_j(t)}(t)$ associated with $s_j(t)$, where $T$ is the time horizon. The adaptive decision rule employed by player $j$ is given by $\rho(\union_{d=1}^{t} \{ r_{s_j(d)}(t), s_j(d) \})$, which is a distributed policy based on its own observation and decision history.

We assume a collision model $\mathcal{M}$ in which player $j$ receives a reward $r_{s_j(t)}(t)$ from arm $s_j(t) $ if it is the only player to select arm $s_j(t)$ at time $t$. Then, the expected total reward can be written as
\begin{equation}\label{exp_re}
	\expt^\rho [S]=\expt^\rho \big[\sum_{t=1}^{T}\sum_{i=1}^{N} \sum_{j=1}^{M}\mu_i(t)\mathcal{O}_{i,j}(t) \big],
\end{equation}
where $\mathcal{O}_{i,j}(t)=1$ if only player $j$ selects arm $i$ at time $t$ and is zero otherwise;  and the expectation is computed over different realization of $\mathcal{O}_{i,j}(t)$ under policy $\rho$.

Let $\sigma_t$ be a permutation of $\until{N}$ at time $t$ such that 
$$\mu_{\sigma_t(1)}(t) \geq \ldots \geq \mu_{\sigma_t(N)}(t).$$
The maximization of $\expt^\rho[S]$ in~\eqref{exp_re} is equivalent to minimizing the expected cumulative group regret defined by
$$R^\rho_\mathcal{M}(T)=\sum_{t=1}^{T}\sum_{k=1}^{M}\mu_{\sigma_t(k)}(t)-\expt^\rho \big[\sum_{t=1}^{T}\sum_{i=1}^{N} \sum_{j=1}^{M}\mu_i(t) \mathcal{O}_{i,j}(t) \big].$$
$R_\mathcal{M}^\rho(T)$ is the expect difference between the cumulative reward of the perfect selection and that of policy $\rho$ under collision model $\mathcal{M}$. Our main purpose here is to design policy $\rho$  that minimizes $R_\mathcal{M}^\rho(T)$.

In this paper, we study the above MAB problem in an abruptly-changing environment, in which the mean rewards from arms switch to unknown values at unknown time instants. We refer to these time instants as \emph{breakpoints}. We assume that the number of breakpoints until time $T$ is $\Upsilon_T \in O(T^\nu)$, where $\nu \in [0,1)$ and is known a priori. In addition, the minimum difference in mean rewards between any pair of arms at any time is lower bounded by $\Delta_{\min}$.


\subsection{Algorithms for the stationary stochastic MAB problem}
In the stationary MAB problem, $\mu_i(t)$ is a constant function of time denoted by $\mu_i$, for each $i\in \until{N}$. In this setting, two algorithms relevant to this paper are : (i) UCB for the single player MAB problem (ii) Distributed Learning algorithm with Prioritization (DLP) for multi-player MAB problem.

The UCB algorithm initializes by selecting each arm once and subsequently selects the arm $s(t)$ at time $t$ defined by 
\[
s(t) = \argmax_{i \in \until{N} } \{\bar r_i (t-1) + c_i(t-1)\}, 
\]
where $\bar r_i (t-1)$ is the statistical mean of the reward from arm $i$ until time $t$, $c_i(t-1)=\sqrt{\frac{2 \ln t }{n_i(t-1)}}$ and $n_i(t-1)$ is the number of times arm $i$ has been selected until time $t$. Auer~\etal \ \cite{PA-NCB-PF:02} showed that UCB algorithm achieves expected cumulative regret that is within a constant factor of the optimal. 

DLP is a distributed algorithm based on the UCB algorithm. Here, distributed algorithm means that a player makes a decision only based upon its own observation history and does not communicate with other players. To avoid regret caused by multiple players selecting the same arm, Selective Learning of the $k$-th largest expected reward (SL($k$)) algorithm is designed aiming at making player $k$ pick the $k$-th best arm. Towards this end, player $k$ determines the set $A_k(t)$ of $k$ arms associated with the $k$ largest values in the set $\setdef{\bar r_i (t-1) + c_i(t-1)}{i \in \until{N}}$, and select arm 
$$s_k(t)=\argmin_{i\in A_k(t)} \{ \bar r_i (t-1) - c_i(t-1) \}.$$
Gai~\etal \ \cite{gai2014distributed} established that DLP achieves order-optimal expected cumulative regret.

\subsection{SW-UCB\# for the non-stationary MAB problem}
The Sliding-Window UCB (SW-UCB) algorithm~\cite{AG-EM:08} was adapted in~\cite{LW-VS:17i} to develop SW-UCB\#. SW-UCB\# is an algorithm for  single player MAB problem in non-stationary environment. At time $t$, it maintains an estimate of the mean reward $\bar{r}_i(t,\alpha)$ at each arm $i$, using only the rewards collected within a sliding-window of observations. Let the width of the sliding-window at time $t \in \until{T}$ be $\tau(t, \alpha) = \min\{\lceil \lambda t^\alpha \rceil, t\}$, where parameter $\alpha \in (0,1]$ is tuned based on environment characteristics. 
Let $n_i(t,\alpha)=\sum_{d= t- \tau(t,\alpha)+1}^{t} \bs{1}_{\{ s(d)=i \}}$ be the number of times arm $i$ has been selected within the time-window at time $t$, where $\bs{1}_{\{\cdot \}}$ is the indicator function, and 
\begin{equation}
	\overline{r}_i(t, \alpha)=\frac{1}{n_i(t, \alpha)} \sum_{d= t- \tau (t, \alpha)+1}^{t} r_i(d) \bs{1}_{\{ s(d)=i \}}.\label{Sample_Mean}
\end{equation}
Based on the above estimate, the SW-UCB\# algorithm at each time selects the arm 
\begin{equation}
s(t) =\argmax_{i \in \until{N}}\{\overline{r}_i(t-1, \alpha) + c_i(t-1, \alpha)\},\label{UCB}
\end{equation}
where  $c_i(t-1, \alpha)=  \sqrt{\frac{(1+\alpha) \ln t}{n_i(t-1, \alpha)}}$. In the following, we would refer to $\overline{r}_i(t-1, \alpha) + c_i(t-1, \alpha)$, for $i \in \until{N}$ as the upper confidence bounds on the estimated rewards. It has been shown in \cite{LW-VS:17i} that the expected cumulative regret for SW-UCB\# under the abruptly-changing environment is upper bounded by a sublinear function of time, i.e., the time average of the regret asymptotically converges to zero.

\section{The Round Robin SW-UCB\# algorithm in Abruptly-Changing Environment}\label{sec:RR-SW-UCB}
In this section, we present the Round Robin SW-UCB\# algorithm (RR-SW-UCB\#) for multi-player MAB problem in abruptly-changing environment.

\subsection{The RR-SW-UCB\# algorithm}
In the RR-SW-UCB\# algorithm, each player maintains the estimate of mean rewards using a sliding window of observations. Each player $k \in \until{M}$ computes $\bar{r}_i^k(t,\alpha)$ and $c_i^k(t,\alpha)$ as in \eqref{Sample_Mean} and \eqref{UCB} using their own observation, and maintains an upper confidence bound on the estimated mean reward $ \bar{r}_{i}^k(t-1,\alpha) +c_i^k(t-1,\alpha) $ from each arm $ {i \in \until{N}}$. For initial $N$ iterations, i.e., $t \in \until{N}$, the player $k$ selects each arm once. Then, at time instants $\seqdef{N+\eta M+1}{\eta \in \integer_{\geq 0}}$, it computes the set $\Omega_k$ containing $M$ arms with $M$ largest values in the set
$$\setdef{\bar{r}_{i}^k(t-1,\alpha) +c_i^k(t-1,\alpha)}{i \in \until{N}}.$$
These arms are sorted in ascending value of their indices (not using the upper confidence bounds), and the player $k$ selects these arms in a round robin fashion starting with the $k$-th arm in the set. It will be shown in the following section that the estimated set of $M$ best arms, $\Omega_k$, will be the same for each player with high probability. Details of RR-SW-UCB\# is shown in Algorithm \ref{algo:RR-SW-UCB}. The free parameter $\lambda$ in the algorithm can be used to refine the finite time performance of the algorithm. 




\IncMargin{.3em}
\begin{algorithm}[t]
	{\footnotesize
		\SetKwInOut{Input}{  Input}
		\SetKwInOut{Set}{  Set}
		\SetKwInOut{Title}{Algorithm}
		\SetKwInOut{Require}{Require}
		\SetKwInOut{Output}{Output}
		
		{\it For abruptly-changing environment} \\   
		
		\Input{$\nu \in [0,1)$, $\Delta_{\min} \in (0,1)$, $\lambda \in \real_{>0}$ , $T\in \natural$ \& $k$\;}
		
		\Set{$\alpha = \frac{1-\nu}{2}$\;}		
		\smallskip
		
%
%
		
		\Output{sequence of arm selection for player $k$\;}
		
		\medskip
		
		\emph{\% Initialization:}
		
		\nl Set $\Omega_k \leftarrow \emptyset$, ordered set $\mathcal{G} \leftarrow ()$, and $t\leftarrow 1$\;
		\smallskip
		
		\nl \While{$t \leq T$}{
			\smallskip
			
			\smallskip
			
			\nl  \If{$t \in \until{N}$ \smallskip}
			{Pick arm $s_k(t)=\text{mod}(t+k-2,N)+1$;
			}
	
			\smallskip
			\smallskip
			
			\nl  \Else
			{Compute $\Omega_k$ containing $M$ arms with $M$ largest values in
				$$\setdef{\bar{r}_{i}^k(t-1,\alpha) +c_i^k(t-1,\alpha)}{i \in \until{N}};$$
					
				Ascending sort the arm indices in $\Omega_k$, $\mathcal{G} \leftarrow \text{sort}_\uparrow(\Omega_k)$;
				\smallskip
				
				\For{$i\in \until{M}$\smallskip}
				{Pick arm $s_k(t)=\mathcal{G} (\text{mod}(t-N+k-2,M)+1);$\\
				$t \leftarrow t+1$\;
				}				
			}
			\smallskip					
		}

		\caption{\textit{The RR-SW-UCB\# Algorithm}}
		\label{algo:RR-SW-UCB}}
\end{algorithm} 

\DecMargin{.3em}

\subsection{Analysis of RR-UCB\#}
Before the analysis, we introduce the following notation. Let $\Omega^M_*(t)$ denote the set of $M$ arms with the $M$ largest mean rewards at time $t$. Then, the total number of times $\Omega_k(t) \ne \Omega^M_*(t)$  until time $T$ can be defined as
$$\mathcal{N}_k(T) := \sum_{t=1}^T \bs{1}_{\{\Omega_k(t) \neq \Omega^M_*(t)\}}.$$
We now upper bound $\mathcal{N}_k(T)$ in the following lemma.

\begin{lemma}\label{lemma:misidentify}
	For the RR-SW-UCB\# algorithm and the multi-player MAB problem with with $N$ arms and $M$ players in the abruptly-changing environment with the number of break points $\Upsilon_T=O(T^\nu)$, $\nu \in [0,1)$, the total number of times $\Omega_k(t) \ne \Omega^M_*(t)$ until time $T$ for any player $k$ satisfies 
	\begin{align*}
	&\mathcal{N}_k (T) \!  \leq  \! (N-M) \Big [ \Big (\frac{T^{1-\alpha}}{\lambda(1-\alpha)}+1 \Big ) \Big ( 1+\frac{4M (1+\alpha) \ln T }{\Delta_{\min}^2} \Big ) \\
	& +\frac{\pi^2}{3}\Big(\frac{\lambda + M + 1}{M} \Big)^2 \Big ] +\Upsilon_T \big (\lceil \lambda (T-1)^\alpha \rceil +M-1 \big )+N.
	\end{align*}
\end{lemma}
\medskip
\begin{proof}
	We begin by separately analyzing windows with and without breakpoints.
	
	\noindent\textbf{Step 1:}
	We define set $\mathcal{T}$ such that for all $t\in\mathcal{T}$, $t$ is either a breakpoint or there exists a break point in its sliding-window of observations $\{t-\tau(t-1,\alpha),\ldots,t-1\}$. For $t\in\mathcal{T}$, the statistical means are biased. It follows that
	$$|\mathcal{T}| \leq \Upsilon_T \lceil \lambda (T-1)^\alpha \rceil.$$
Consequently, $\mathcal{N}_k(T)$ can be upper-bounded as
	\begin{equation}\label{eq:boundN1}
		\mathcal{N}_k(T) \leq \Upsilon_T \Big (\lceil \lambda (T-1)^\alpha \rceil+M-1 \Big ) + \tilde{\mathcal{N}}_k(T),
	\end{equation}
	where $\tilde{\mathcal{N}}_k(T) := \sum_{t=1}^T \bs{1}_{\{ \Omega_k(t) \neq \Omega^M_*(t), t \notin \mathcal{T}\}}$. The term $M-1$ in~\eqref{eq:boundN1} is due to the fact that $\Omega_k$ is computed every $M$ steps. In the following steps, we will bound $\tilde{\mathcal{N}}_k(T)$.
	
	\noindent\textbf{Step 2:}
	If $\Omega_k(t) \neq \Omega^M_*(t)$, there exists at least one arm $i$ such that $i \in \Omega_k(t)$ and $i \notin \Omega^M_*(t)$. Then, it can be shown that
	\begin{align}
	&\tilde{\mathcal{N}}_k(T) \leq  \! N \! + \!\!\!\! \sum_{t=N+1}^T \sum_{i=1}^{N} \bs{1}_{\{i \in \Omega_k(t), i \notin \Omega^M_*(t), t\notin 	\mathcal{T}, n_i(t-1,\alpha)< l(t,\alpha) \}} \nonumber \\
	& +\sum_{t=N+1}^T \sum_{i=1}^{N} \bs{1}_{\{i \in \Omega_k(t), \, i \notin \Omega^M_*(t), t\notin \mathcal{T}, n_i(t-1,\alpha) \geq l(t,\alpha) \}},\label{eq:boundN2}		
	\end{align}
	where we choose $l(t,\alpha)=\frac{4 (1+\alpha) \ln t }{\Delta_{\min}^2}$.
	
	We begin with bounding the second term on the right hand side of inequality \eqref{eq:boundN2}. First, we partition time instants into $G$ epochs.	Let $G \in \natural$ be such that
	\begin{equation}
	[\lambda(1-\alpha)(G-1)]^{\frac{1}{1-\alpha}}< T \leq [\lambda(1-\alpha) G]^{\frac{1}{1-\alpha}}.
	\label{eq:G}
	\end{equation}
	Then, we have the following epochs
	\begin{equation}\label{partition}
	\{1+\phi(g-1),\ldots,\phi(g)\}_{g \in \until{G}},
	\end{equation}
	where $\phi(g) = \big \lfloor [\lambda(1-\alpha) g]^{\frac{1}{1-\alpha}} \big \rfloor$. Let $\tilde{t}$ be any time instant other than the first instant in the $g$-th epoch. We will now show that all but one of the time instants in the $g$-th epoch until $\tilde{t}$ are ensured to be contained in the time-window at $\tilde{t}$. 
	Towards this end, consider the increasing convex function $f(x)=x^{\frac{1}{1-\alpha}}$ with $\alpha \in (0,1)$. It follows that $f(x_2)-f(x_1) \leq f'(x_2)(x_2-x_1)$ if $x_2 \geq x_1$. Then, substituting $x_1= g-1$ and $x_2 = \frac{\tilde{t}^{1-\alpha}}{\lambda(1-\alpha)}$ in the above inequality and simplifying, we get 	
	\begin{equation*}
		\tilde{t} - (\lambda(1-\alpha)(g-1))^{\frac{1}{1-\alpha}}\leq \lambda \tilde{t}^\alpha \Big(\frac{\tilde{t}^{1-\alpha}}{\lambda(1-\alpha)}-g+1\Big).
	\end{equation*}
	Since by definition of the $g$-th epoch, $\frac{\tilde{t}^{1-\alpha}}{\lambda(1-\alpha)} \leq g$, we have
	\begin{align*}
		\tilde{t}-\lfloor (\lambda (1-\alpha) (g-1))^{\frac{1}{1-\alpha}}\rfloor		
		&\leq \min\{ \tilde t +1, \lambda \lceil \tilde{t}^{\alpha}\rceil+1 \}  \\
		& =  \tau(\tilde{t},\alpha)+1. \label{window}
	\end{align*}
	The only time instant in $g$-th epoch that is possibly not contained in time window at $\tilde{t}$ is $1+\phi(g-1)$. Then for any arm $i \in \until{N}$,
	\begin{equation}\label{ieq:windowbound}
		\sum_{2+\phi(g-1)}^{\tilde{t}} \bs{1}_{\{ i \in \Omega_k(t) \}} \leq Mn_i(\tilde{t}).
	\end{equation}
	
Furthermore, in the $g$-th epoch in the partition, either
	$$\sum_{t \in g \text{-th epoch}} \sum_{i=1}^{N} \bs{1}_{\{ i \in \Omega_k(t), \, i \notin \Omega^M_*(t),\, t\notin \mathcal{T},\, n_i(t-1,\alpha)< l(t,\alpha) \}} = 0, $$
	or there exist at least one time-instant $t$ in the $g$-th epoch such that
	\begin{equation*}
		\sum_{i=1}^{N} \bs{1}_{\{i \in \Omega_k(t), \, i \notin \Omega^M_*(t),\, t\notin \mathcal{T},\, n_i(t-1,\alpha)< l(t,\alpha) \}} > 0. 
	\end{equation*}	
	Let the last time instant satisfying this condition in the $g$-th epoch be
	\begin{align*}
		t(g) = \max \{ & t \in g \text{-th epoch}| \\
		&\sum_{i=1}^{N} \bs{1}_{\{ i \in \Omega_k(t), \, i \notin \Omega^M_*(t),\, t\notin \mathcal{T},\, n_i(t-1,\alpha)< l(t,\alpha) \}} > 0 \}.
	\end{align*}
	Note that $t(g)\notin \mathcal{T}$ indicates, for each $i \in \until{N}$, $\mu_i(s)$ is a constant for all $s \in \{t(g)-\tau 
	( t(g)-1,\alpha ),\ldots,t(g)\}$.	
	Then, it follows from~\eqref{ieq:windowbound} that 
	\begin{align}\label{ieq:inepoch}
	&\sum_{t \in g\text{-th epoch}}\sum_{i=1}^{N} \bs{1}_{\{i \in \Omega_k(t), \,i \notin \Omega^M_*(t), t\notin \mathcal{T}, n_i(t-1,\alpha) < l(t,\alpha) \}} \nonumber \\
	&\leq \! N \! - \! M + \!\!\!\!\!\!\!\!\!  \sum_{t=\phi(g-1)+2}^{t(g)} \sum_{i=1}^{N} \bs{1}_{\{ i \notin \Omega^M_*(t), t\notin \mathcal{T}, n_i(t-1,\alpha) < l(t,\alpha) \}} \nonumber \\
	&\leq N-M + \sum_{i=1}^N M l(t_i(g),\alpha)\bs{1}_{\{ i \notin \Omega^M_*(t(g)),\, t(g) \notin \mathcal{T} \}} \nonumber \\
	&\leq (N-M) \Big ( 1 + \frac{4M (1+\alpha) \ln T }{\Delta_{\min}^2} \Big ),
	\end{align}
	where 	
	$
		t_i(g) = \max \setdef{t \in g \text{-th epoch}}{i \in  \Omega_k(t), \, i \notin \Omega^M_*(t), \, t\notin \mathcal{T}, \, n_i(t-1,\alpha)< l(t,\alpha) } 
	$
	and $t_i(g) \leq t(g)$ for all $i \in \until{N}$. Therefore, from \eqref{partition} and \eqref{ieq:inepoch} , we have
	\begin{align}\label{eq:boundN3}
		&\sum_{t=N+1}^T \sum_{i=1}^{N} \bs{1}_{\{i \in \Omega_k(t), i \notin \Omega_*(t),n_i(t-1,\alpha)< l(t,\alpha) \}} \nonumber\\
		&\leq  G(N-M) \Big (1+ \frac{4 M  (1+\alpha) \ln T }{\Delta_{\min}^2} \Big ).
	\end{align}
		
	\noindent \textbf{Step 3:}
	In this step, we bound the expectation of the last term in \eqref{eq:boundN2}. It can be shown that
	\begin{align}
	&\sum_{i=1}^{N} \bs{1}_{\{ i \in \Omega_k(t), i \notin \Omega^M_*(t), t\notin \mathcal{T}, n_i(t-1,\alpha) \geq l(t,\alpha)  \}} \nonumber \\
	&\leq\!\!\!\! \sum_{i \notin \Omega^M_*(t)} \sum_{\zeta \in \Omega^M_*(t)}  \sum_{s_\zeta=1}^{h(t)}\sum_{s_i=l(t,\alpha)}^{h(t)} \!\!\!\!\! \bs{1}_{\{n_\zeta(t-1,\alpha)=s_\zeta, \, n_i(t-1,\alpha)=s_i , t\notin \mathcal{T}\}} \label{a} \nonumber \\ 
	& \times \bs 1_{\{\overline{r}_{\zeta}(t-1,\alpha) +c_\zeta(t-1,\alpha) \leq \overline{r}_{i}(t-1,\alpha) +c_{i}(t-1,\alpha), n_i(t-1,\alpha) \geq l(t,\alpha)\} },
	\end{align}
	where $h(t) :=\Big \lceil \frac{\lceil \lambda (t-1)^ \alpha \rceil}{M} \Big \rceil$ is the maximum number of times an arm can be selected within the time window at $t-1$. Note that $\overline{r}_{j}(t-1,\alpha) + c_j(t-1,\alpha) \leq \overline{r}_i(t-1,\alpha) +c_i(t-1,\alpha)$ means at least one of the following holds.
	\begin{align}
	\overline{r}_{i}(t-1,\alpha) &\geq  \mu_i(t) + c_i(t-1,\alpha), \label{b}\\
	\overline{r}_{\zeta}(t-1,\alpha) &\leq  \mu_{\zeta}(t) - c_{\zeta}(t-1,\alpha), \label{c}\\
	\mu_{\zeta}(t) - \mu_i (t) & < 2 c_i(t-1,\alpha). \label{d}
	\end{align}
	Since $n_i(t-1,\alpha) \geq l(t,\alpha)$, (\ref{d}) does not hold. Applying Chernoff-Hoeffding inequality~\cite[Theorem 1]{WH:63} to bound the probability of events (\ref{b}) and (\ref{c}), we obtain
	\begin{align}
	\mathbb{P}( \overline{r}_{i}(t-1,\alpha) &\geq  \mu_i(t) + c_i(t-1,\alpha) )\leq t^{-2(1+\alpha)} \label{e},\\
	\mathbb{P}( \overline{r}_{j}(t-1,\alpha) &\leq  \mu_{j}(t) -  c_{j}(t-1,\alpha) )\leq t^{-2(1+\alpha)}. \label{f}
	\end{align}
	Since $\Omega_k$ is only computed at time instants $\seqdef{N+\eta M+1}{\eta \in \integer_{\geq 0}}$, from \eqref{a}, \eqref{e} and \eqref{f}, we have
	\begin{align}\label{eq:boundN4}
	&\mathbb{E} \Big[ \sum_{t=N+1}^T \sum_{i=1}^{N} \bs{1}_{\{i \in \Omega_k(t), i \notin \Omega^M_*(t), t\notin \mathcal{T}, n_i(t-1,\alpha) \geq l(t,\alpha) \}} \Big]  \nonumber \\
    &\leq (N-M)M \sum_{s_\zeta=1}^{h(f(\eta))} \sum_{s_i=l(t,\alpha)}^{h(f(\eta))} \sum_{\eta=0}^{\lceil \frac{T-N}{M} \rceil }2M f(\eta)^{-2(1+\alpha)} \nonumber \\
	&\leq (N-M)M^2 \sum_{\eta=0}^{\lceil \frac{T-N}{M} \rceil }2f(\eta)^{-2(1+\alpha)} h(f(\eta))^2  \nonumber \\
	&\leq (N-M) \Big(\frac{\lambda + M + 1}{M} \Big)^2 \sum_{\eta=1}^{\infty}2  \eta^{-2} \nonumber\\
	&= \frac{\pi^2}{3} (N-M)\Big(\frac{\lambda + M + 1}{M} \Big)^2 ,
	\end{align}
	where $f(\eta):=N+\eta M+1$. Therefore, it follows from (\ref{eq:boundN1}), (\ref{eq:boundN2}), (\ref{eq:boundN3}), and (\ref{eq:boundN4}) that
	\begin{align*}
	\mathcal{N}_k &(T)  \leq N+ (N-M) \Big [ G \Big ( 1+\frac{4M (1+\alpha) \ln T }{\Delta_{\min}^2} \Big ) \\
	& +\frac{\pi^2}{3}\Big(\frac{\lambda + M + 1}{M} \Big)^2 \Big ] +\Upsilon_T \big (\lceil \lambda (T-1)^\alpha \rceil +M-1 \big ).
	\end{align*}
	From (\ref{eq:G}), we have $G\leq \frac{T^{1-\alpha}}{\lambda(1-\alpha)}+1$, and this yields the conclusion.	
\end{proof}

Based on Lemma~\ref{lemma:misidentify}, we now establish the order of expected cumulative group regret of RR-SW-UCB\# in the abruptly changing environment.
\begin{theorem}\label{theorem:RR-SW-UCB bound}
	For the RR-SW-UCB\# algorithm and the multi-player MAB problem with $N$ arms and $M$ players in the abruptly-changing environment with the number of break points $\Upsilon_T=O(T^\nu)$, $\nu \in [0,1)$, under collision model $\mathcal{M}$, the expected cumulative group regret satisfies 
	$$R_\mathcal{M}^{\text{RR-SW-UCB\#}}(T)\in O\big(T^{\frac{1+\nu}{2}}\ln T \big)$$
\end{theorem}
\begin{proof}
	If all player identify $\Omega_*^M(t)$ correctly at time $t$, no expected regret is accrued. It follows from Lemma~\ref{lemma:misidentify} that $\mathcal{N}_k(T) \in O( T^{\frac{1+\nu}{2}}\ln T )$ for all $k \in \until{M}$. The total number of times that any player misidentifies $\Omega_*^M(t)$ until time $T$ can be upper bounded by  $\sum_{k=1}^M \mathcal{N}_k(T)$. Thus, we conclude the proof.
\end{proof}

\section{The SW-DLP algorithm in Abruptly-Changing Environment}\label{sec:SW-DLP}
In this section, we introduce the SW-DLP algorithm for multi-player MAB problem in abruptly-changing environment.
\subsection{The SW-DLP algorithm}
The SW-DLP algorithm combines the ideas from SW-UCB\# and DLP. Similar to SW-UCB\#, upper confidence bounds on the mean rewards are computed using a sliding window of observations as in \eqref{Sample_Mean} and \eqref{UCB}. Following the same allocation rule in DLP, at each time instant, player $k$ computes a set $A_k$ containing all $k$ largest values in the set
$$\setdef{\bar{r}_{i}^k(t-1,\alpha) +c_i^k(t-1,\alpha)}{i \in \until{N}},$$
and selects arm 
$$s_k(t)=\argmin_{i\in A_k(t)} \{ \bar{r}_{i}^k(t-1,\alpha) - c_i^k(t-1,\alpha) \}.$$
Details about the SW-DLP is shown in Algorithm \ref{algo:SW-DLP}. The parameters in the SW-DLP algorithm are the same as in the RR-SW-UCB\# algorithm. In the following, we will refer to $\bar{r}_{i}^k(t-1,\alpha) - c_i^k(t-1,\alpha) $ as the lower confidence bound on the estimate reward from arm $i$.

\IncMargin{.3em}
\begin{algorithm}[t]
	{\footnotesize
		\SetKwInOut{Input}{  Input}
		\SetKwInOut{Set}{  Set}
		\SetKwInOut{Title}{Algorithm}
		\SetKwInOut{Require}{Require}
		\SetKwInOut{Output}{Output}
		
		{\it Input, output and parameters are the same as RR-SW-UCB\#} \\
		
		\emph{\% Initialization:}\\
		\nl Set $A_k \leftarrow \emptyset$ and $t \leftarrow 1$\;
		\smallskip
		
		\nl \While{$t \leq T$}{
			\smallskip
			
			\smallskip
			
			\nl  \If{$t \in \until{N}$ \smallskip}
			{Pick arm $s_k(t)=\text{mod}(t+k-2,N)+1$;
			}
			
			\smallskip
			
			\smallskip
			
			\nl  \Else
			{Compute $A_k$ containing $k$ arms with $k$ largest values in
				$$\setdef{\bar{r}_{i}^k(t-1,\alpha) +c_i^k(t-1,\alpha)}{i \in \until{N}};$$
											
				Pick arm
				$$s_k(t)=\argmin_{i\in A_k} \{\bar{r}_{i}^k(t-1,\alpha) - c_i^k(t-1,\alpha)\};$$
			}
			\smallskip					
		}

		\caption{\textit{The SW-DLP Algorithm}}
		\label{algo:SW-DLP}}
\end{algorithm} 

\DecMargin{.3em}
\subsection{Analysis of SW-DLP}

\begin{theorem}
	For the SW-DLP algorithm and the multi-player MAB problem with $N$ arms and $M$ players in the abruptly-changing environment with the number of break points $\Upsilon_T=O(T^\nu)$, $\nu \in [0,1)$, under collision model $\mathcal{M}$, the expected cumulative group regret satisfies 
	$$R_\mathcal{M}^{\text{SW-DLP}}(T)\in O\big ( T^{\frac{1+\nu}{2}}\ln T \big)$$
\end{theorem}
\begin{proof}
The proof is similar to the proof of Theorem~\ref{theorem:RR-SW-UCB bound} and we only present a sketch. 
Let $\theta_k(t)$ be the $k$-th best arm at time $t$. The total number of time instants that $\theta_k(t)$  is not selected by player $k$ with SW-DLP satisfies
	\begin{equation}\label{ieq:misidentify1}
		\hat{\mathcal{N}}_k \leq \Upsilon_T \lceil \lambda (T-1)^\alpha \rceil + \sum_{t=1}^{T} \bs{1}_{\{ s_k(t) \neq \theta_k(t), \, t \notin \mathcal{T}\}},
	\end{equation}
	where $\hat{\mathcal{N}}_k := \sum_{t=1}^{T} \bs{1}_{\{ s_k(t) \neq \theta_k(t)\}}$.
	
	We partition the time instants as in \eqref{partition}. Then, similarly to~\eqref{ieq:windowbound} in the proof of Lemma~\ref{lemma:misidentify}, it can be shown that
	\begin{equation}\label{ieq:windowbound2}
	\sum_{2+\phi(g-1)}^{\tilde{t}} \bs{1}_{\{ s_k(t) = i \}} \leq n_i(\tilde{t}),
	\end{equation}
	for any arm $i \in \until{N}$ and $\tilde{t} \in g\text{-th epoch}$ and $t \ne 1+\phi(g-1)$.
	
We study the event that player k does not select arm $\theta_k(t)$ at time $t$ under two scenarios: (i) $A_k \neq \Omega_*^k(t)$, and (ii) $A_k = \Omega_*^k(t)$, where $\Omega_*^k(t)$ is the set with $k$ best arms at time $t$. Then, we have
	\begin{align}\label{ieq:misidentify2}
		\sum_{t=1}^{T}   \bs{1}_{\{s_k(t)\neq \theta_k(t), \, t \notin \mathcal{T} \}}  \leq &\sum_{t=1}^{T}    \bs{1}_{\{ A_k \neq \Omega_*^k(t), \, t \notin \mathcal{T} \}} \nonumber \\
		 +  \sum_{t=1}^{T} &\bs{1}_{\{s_k(t)\neq \theta_k(t), A_k = \Omega_*^k(t), \, t \notin \mathcal{T} \}} .
	\end{align}
    
 Note that unlike  RR-SW-UCB\#, in SW-DLP, after initialization, $A_k(t)$ is computed every time instead of only at time instants $\seqdef{N+\eta M+1}{\eta \in \integer_{\geq 0}}$. However, this difference does not  change the order of the total number of times that $\Omega^k_*(t)$ is misidentified. Therefore, with the fact~\eqref{ieq:windowbound2}, it follows similarly to the proof of Lemma~\ref{lemma:misidentify} that
	\begin{equation}\label{ieq:misidentify2_1}
		\sum_{t=1}^{T} \bs{1}_{\{ A_k \neq \Omega_*^k(t) \}} \in O\big( T^{\frac{1+\nu}{2}}\ln T \big).
	\end{equation}

The Chernoff-Hoefding inequality is symmetric about the estimated mean and upper tail bound is identical to the lower tail bound. Hence, second term on the right hand side of inequality~\eqref{ieq:misidentify2} that involves selecting $s_k(t)$ using lower confidence bounds can be bounded similarly to the first term. 
Thus, we have
	\begin{equation}\label{ieq:misidentify2_2}
		\sum_{t=1}^{T} \bs{1}_{\{ s_k(t)\neq \theta_k(t),\, A_k = \Omega_*^k(t), \, t \neq \mathcal{T}\}} \in O\big( T^{\frac{1+\nu}{2}}\ln T \big).
	\end{equation}
	Substituting~\eqref{ieq:misidentify2_1} and~\eqref{ieq:misidentify2_2} into~\eqref{ieq:misidentify2}, and substituting~\eqref{ieq:misidentify2} into~\eqref{ieq:misidentify1}, we conclude that $\hat{\mathcal{N}}_k \in O\big( T^{\frac{1+\nu}{2}}\ln T \big)$.

The number of times the group does not receive a reward from arm $\theta_k(t)$ is upper bounded by the number of times player $k$ does not receive a reward from arm $\theta_k(t)$. Player $k$ does not receive a reward from arm $\theta_k(t)$ if one of the following conditions is true (i) arm $\theta_k(t)$ is not selected by player $k$, and (ii) arm $\theta_k(t)$ is selected by another player $j \neq k$.  The total number of times either one of these events occurs at any arm $\theta_k(t)$, for all $k\in\until{M}$, can be upper bounded by $\sum_{k=1}^M 2 \hat{\mathcal{N}}_k$. Since $\hat{\mathcal{N}}_k \in O\big( T^{\frac{1+\nu}{2}}\ln T \big)$ for all $k \in \until{M}$, we conclude the proof.
\end{proof}

\begin{remark}[\bit{Comparison of RR-SW-UCB\# and SW-DLP}]
In multi-player MAB algorithms, assignment of a player to a targeted arm is crucial to avoid collisions. In RR-SW-UCB\#, the indices of arms and the indices of players are employed for this assignment. A round robin policy ensures all players select $M$-best arms persistently and have an accurate estimate of the associated mean rewards. While in SW-DLP, such accurate estimation by all players is driven by the lower confidence bound based assignment of players to the arms.
	\oprocend
\end{remark}

\section{Numerical Illustration}\label{sec:numerical}
In this section, we present simulation results for RR-SW-UCB\# and SW-DLP in the abruptly changing environment. In the simulations, we consider a multi-player MAB problem with $6$ arms and $3$ players. 
We consider three different values $\{0.15,\, 0.3,\, 0.45\}$ of parameter $\nu$ that describes the number of breakpoints to show the performance the both algorithms. The breakpoints are introduce at time instants where the next element of sequence $\seqdef{\lfloor t^\nu \rfloor}{t \in \until{N}}$ is different from current element. We pick them at these time instants to make number of breakpoints $\Upsilon_t \in O(t^\nu)$ uniformly for all $t \in \until{T}$. At each break point, the mean rewards at each arm is randomly selected from $\{0.05,\, 0.22,\, 0.39,\, 0.56,\, 0.73,\, 0.90\}$. In both algorithms, we select $\lambda=12.3$.

As shown in Fig.~\ref{fig:performance}, with either algorithm, the ratio of the empirical cumulative group regret to the order of $t^{\frac{1+\nu}{2}}\ln t$ is upper bounded by a constant. The dashed lines in Fig.~\ref{fig:performance} (b) are taken directly from (a). The comparison shows that the cumulative regret of RR-SW-UCB\# is much lower than SW-DLP. However, if cost of switching between arms is considered, then the round-robin structure of RR-SW-UCB\# would incur significant cost, and in such a scenario SW-DLP might be preferred.


\begin{figure} 
	\centering
	 \subfigure[RR-SW-UCB\#]{ \label{fig:RR-UCB} 
		 \includegraphics[width=0.475\linewidth]{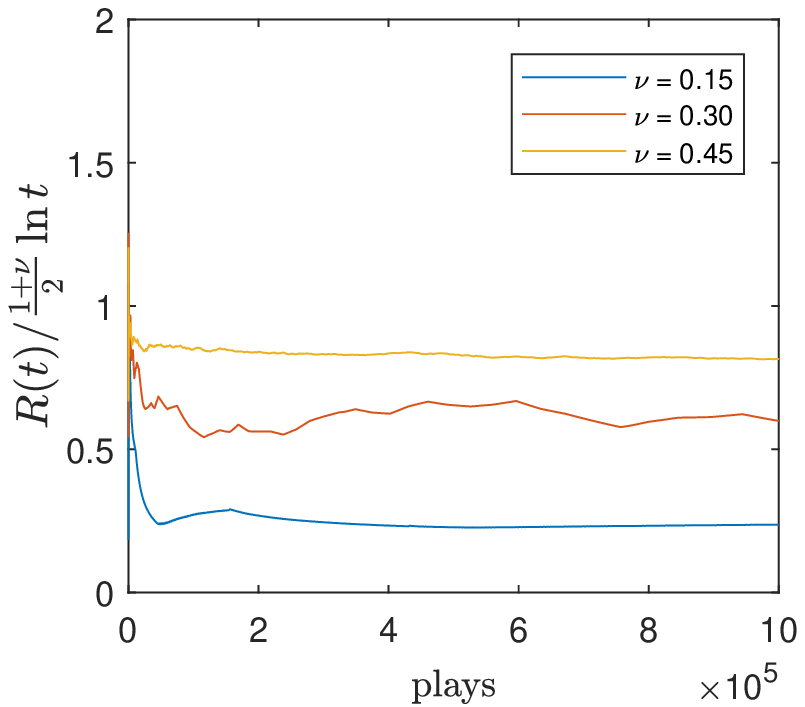}}
	 \subfigure[SW-DLP and RR-SW-UCB\#]{ \label{fig:SW-DLP} 
		 \includegraphics[width=0.475\linewidth]{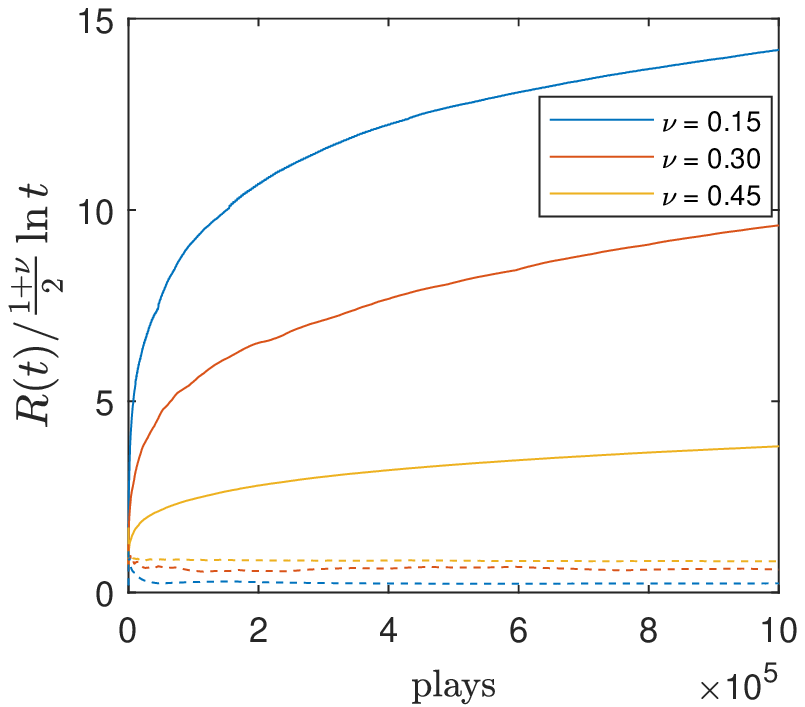}} \caption{The performance the two algorithms in the abruptly changing environment. (a)  The solid lines show the performance of RR-SW-UCB\# (b) The solid lines and dashed lines show the performance of SW-DLP and RR-SW-UCB\#, respectively.} \label{fig:performance} 
\end{figure}

\section{Conclusion and Future Directions}\label{sec:conclusions}
We studied the multi-player stochastic MAB problem in abruptly environment under a collision model in which a player receives a reward by selecting an arm if it is the only player to select that arm. We designed two novel algorithms, RR-SW-UCB\# and SW-DPL to solve this problem.  We analyzed these algorithms and characterized their performance in terms of expected cumulative group regret. In particular, we showed that these algorithms incur sublinear expected cumulative regret, i.e., the time average of the regret asymptotically converges to zero. 

There are several possible avenues for future research. In this paper, we focused on an abruptly changing environment. Extension of this work to other classes of non-stationary environments such as slowly-varying environment is of considerable interest. Another avenue of future research is extension of these algorithms to the multi-player Markovian MAB problem. 


 
\footnotesize 

\bibliographystyle{IEEEtran}
\bibliography{IEEEabrv,bandits,surveillance,mybib}

\end{document}